\documentclass{article} 
\usepackage{iclr2027_conference,times}

\usepackage{amsmath,amsfonts,bm}

\def\eqref#1{equation~\ref{#1}}

\def\1{\bm{1}}

\DeclareMathAlphabet{\mathsfit}{\encodingdefault}{\sfdefault}{m}{sl}
\SetMathAlphabet{\mathsfit}{bold}{\encodingdefault}{\sfdefault}{bx}{n}

\newcommand{\R}{\mathbb{R}}

\newcommand{\KL}{D_{\mathrm{KL}}}

\newcommand{\one}{\mathbf{1}}
\newcommand{\diag}{\operatorname{diag}}

\usepackage{hyperref}
\usepackage{graphicx}
\usepackage{tabularx}
\usepackage{url}
\usepackage{amsthm}

\newtheorem{lemma}{Lemma}[section]

\title{Unlocking Geodesic Gromov-Wasserstein Distances for 3D Modeling}

\author{
Ananya Parashar, Derek Long, Dwaipayan Saha\\
Department of Industrial Engineering and Operations Research \\
Columbia University, New York, NY 10027\\
\texttt{\{ap4658,dl3538,ds4386\}@columbia.edu}
\And
Krzysztof Choromanski\\
Department of Industrial Engineering and Operations Research \\
Columbia University, New York, NY 10027\\
\texttt{\{ap4658,dl3538,ds4386\}@columbia.edu}
}

\author{
Krzysztof Choromanski\textsuperscript{$1,2$ \thanks{equal contribution} $\,\,$\thanks{senior lead}$\,\,$}, Derek Long\textsuperscript{$1\,^{*}$}, Ananya Parashar\textsuperscript{$1\,^{*}$}, Dwaipayan Saha\textsuperscript{$1\,^{*}$}
\vspace{1.3mm}\\
\normalfont
\hspace{3.5cm}
\textsuperscript{$1$}Columbia University, 
\textsuperscript{$2$}Google DeepMind
\vspace{-1.3mm}
}

\iclrfinalcopy 
\begin{document}

\maketitle

\begin{abstract}
\textit{Gromov-Wasserstein Distances} (GWDs) provide quantitative ways of comparing probabilistic distributions defined on different metric spaces by applying techniques from the optimal transport theory. As such, GWD can be potentially useful in a large variety of applications ranging from graph matching problems to 3D object detection. However its practical use at scale is significantly limited by cubic time complexity computations involving dense intra-space distance matrices. Even though in the Euclidean metric spaces several techniques (e.g. involving scalable kernel methods) were proposed to address it, to the best of our knowledge, analogous techniques for general geodesic distances on manifolds, or shortest-path distance on graphs in their discretized variants, were not developed. In this paper, we present \textbf{E}fficient \textbf{G}eodesic \textbf{Gro}mov-\textbf{W}asserstein methods (EGGroW), a new class of efficient algorithms designed to calculate geodesic Gromov-Wasserstein distances with entropic Sinkhorn-like approaches, leveraging recently introduced \textit{GenusSink} methods \citep{genussink} and the theory of random features. We provide important downstream applications, namely: 3D pose estimation and 3D template detection. In the latter setting, we formulate a partial 3D template recovery as a staged problem: capacity-constrained scene selection is followed by semi-relaxed recovery of template visibility and correspondence. Our empirical findings show that EGGroW provides accurate solutions when standard Euclidean-based techniques fail and is characterized by light computational footprint, as our theoretical analysis predicts.
\end{abstract}

\vspace{-3mm}
\section{Introduction \& Related Work}
\label{sec:introduction}

Wasserstein distance \citep{pereira,montesuma2023recent,peyre2019computational,peyre2025otml,IG-OTP, so-otp, Figalli2010TheOP, villani, ambrosio, Villani2003TopicsIO} is a rigorous and one of the most commonly used in machine learning approach to measuring dissimilarity between two probabilistic measures: $\mu$ and $\nu$, defined in a given metric space $\mathcal{X}$. For $p \in [1,+\infty]$, the \textit{Wasserstein $p$-distance} between $\mu$ and $\nu$ with finite $p$-moments is given by the following formula, where $d:\mathcal{X} \times \mathcal{X} \rightarrow \mathbb{R}$ is the corresponding distance function: 
\begin{equation}
\label{eq:wasserstein}
W_{p}(\mu,\nu) \ \inf_{\gamma \in \Gamma(\mu, \nu)}\left(\mathbb{E}_{x,y \sim \gamma} d(x,y)^{p}\right)^{\frac{1}{p}}.    
\end{equation}
Here $\Gamma(\mu,\nu)$ stands for the set of all couplings of $\mu$ and $\nu$. We define a \textit{coupling} as a joint probablity measure on $\mathcal{X} \times \mathcal{X}$ whose marginals on the first and second factor are $\mu$ and $\nu$ respectively, i.e. the following holds for every measurable $\mathcal{A}$:
\begin{equation}
\label{eq:coupling}
\gamma(\mathcal{A} \times \mathcal{X}) = \mu(\mathcal{A}), \textrm{  and } \gamma(\mathcal{X} \times \mathcal{A}) = \nu(\mathcal{A}).    
\end{equation}

In most practical applications of the Optimal Transport, its discretized version is considered, with $\mu$ and $\nu$ supported by $m$ and $n$ points respectively. In this setting, each solution (transport plan) is described by the \textit{transport matrix} $\mathbf{T} \in \mathbb{R}^{m \times n}$. The costs of moving mass between points is encoded by the \textit{cost matrix} $\mathbf{C} \in \mathbb{R}^{m \times n}$, where $C_{i,j}$ is the distance between ith atom of $\mu$ and jth atom of $\nu$. The optimization problem takes the following form for $\mathbf{a} \in \mathbb{R}^{m}_{\geq 0},\mathbf{b} \in \mathbb{R}^{n}_{\geq 0}$ encoding 
$\mu$ and $\nu$ ($\sum_{i=1}^{m}a_{i}=1$, $\sum_{j=1}^{n} b_{j}=1$) and furthermore, $\odot$ is a \textit{Hadamard} (element-wise) product:
\begin{equation}
\label{eq:kantorovich_discrete}
\min_{\mathbf{T} \in \mathbb{R}^{m \times n}} \mathbf{T} \odot \mathbf{C} \textrm{  st.  } \mathbf{T}\mathbf{1}_{n} = \mathbf{a}, \textrm{  } \mathbf{T}^{\top}\mathbf{1}_{m}=\mathbf{b}, \textrm{  } \mathbf{T} \geq 0.   
\end{equation}

Matrix $\mathbf{T}^{*} \in \mathbb{R}^{m \times n}$ takes quadratic space $O(nm)$.
Standard linear-programming algorithms, used to solve Optimal Transport problem from Eq. \ref{eq:kantorovich_discrete}, are of cubic time complexity. One of the most popular relaxations of the problem, providing significant computational speedups, is the so-called \textit{Sinkhorn} (entropic) variation \citep{sinkhorn-1, sinkhorn-survey, sinkhorn-composed}. It adds an entropic regularization term, so that the optimization problem becomes:
\begin{equation}
\label{eq:entropic_form}
\min_{\mathbf{T} \in \mathbb{R}^{n \times n}} \mathbf{T} \odot \mathbf{C} + \epsilon \sum_{i,j}T_{i,j}(\log(T_{i,j})-1).  
\end{equation}
Sinkhorn version can be solved in time $O(nm)$ by the following iterative procedure (with the number of iterations treated as a constant), with \textit{kernel matrix} $\mathbf{K}_{\epsilon}=\exp(-\frac{\mathbf{C}}{\epsilon})$ ($\exp$ and vector-divisions applied element-wise):
\begin{equation}
\mathbf{v}^{\textrm{init}}=\mathbf{1}_{n}, \mathbf{u} \leftarrow \frac{\mathbf{a}}{\mathbf{K}_{\epsilon}\mathbf{v}},
\mathbf{v} = \frac{\mathbf{b}}{(\mathbf{K}_{\epsilon})^{\top}\mathbf{u}}    
\end{equation}
Thus Sinkhorn approach makes the Optimal Transport Problem much more computationally tractable, especially for probabilistic distributions supported on massive size discrete sets.

\textit{Gromov-Wasserstein Distance} (GWD: see Sec. \ref{sec:background} for exact definition) \citep{Memoli2011-rm,peyre2016gromov, xu2019gromov, Vayer2020-rr} is a natural extension of the OT formulation to the scenario, where the measures $\mu$ and $\nu$ are defined on two different metric spaces $\mathcal{X}$ and $\mathcal{Y}$. Importantly $\mathcal{X}$ and $\mathcal{Y}$ do not need to be related to each other in any way, since the new formulation does not rely on any notion of inter-distances between the two spaces. Instead, it evaluates the soft coupling provided by the transport matrix $\mathbf{T}$ by comparing intra-space relationships between the points supporting measures. This feature makes it a potentially very powerful tool in a variety of ML applications, ranging from graph matching problems to 3D object detection.

However practical use at scale of the GWD techniques is significantly limited by cubic time complexity computations involving dense intra-space distance matrices \citep{chowdhury2021quantized}.
Even though in certain settings including costs given via Euclidean distances (e.g. squared Euclidean distances, see: \citep{LT-GWD}), several efficient methods were proposed (e.g. involving scalable kernel methods), to the best of our knowledge, analogous analogous techniques for general geodesic distances on manifolds, or shortest-path distance on graphs in their discretized variants, were not developed. This more general setting, the key to unlocking GWD for 3D modeling (e.g. with meshes and point clouds, where distances are often defined via shortest paths in the related graphs), is the subject of this paper. 

One of the biggest challenges in the geodesic setting is that costs matrices $\mathbf{C}$ seem to be much less structured, at first glance prohibiting computationally efficient operations on them (or their derivatives, such as kernel matrices) \citep{parashar2026manifold}. The recently proposed GenusSink algorithm \citep{genussink} addresses this question in the regular Optimal Transport setting, providing efficient methods for matrix-vector multiplications involving kernel matrices in the geodesic setting. The only assumption (satisfied in most practical applications) is that 
graphs under considerations need to have bounded genus \citep{genus-1}.

In this paper we significantly extend those techniques to computations of the Gromov-Wasserstein Distance. We present \textbf{E}fficient \textbf{G}eodesic \textbf{Gro}mov-\textbf{W}asserstein methods (EGGroW), a new class of efficient algorithms designed to calculate geodesic Gromov-Wasserstein distances with entropic approaches, leveraging the aforementioned \textit{GenusSink} methods and the theory of random features. We provide important downstream applications, namely: 3D pose estimation and 3D template detection. Another contribution of this paper is the formulation of the 3D template detection in the given point cloud scene as a derivative of the original GWD problem. To be more specific, we formulate partial 3D template recovery as a staged problem: capacity-constrained scene selection is followed by semi-relaxed recovery of template visibility and correspondence with the use of GWDs. Our empirical findings show that EGGroW provides accurate solutions when standard Euclidean-based techniques fail and is characterized by light computational footprint, as our theoretical analysis predicts.

Two branches that EGGroW relies on: the theory of GenusSink and random features techniques, are complementary. The former gives an exact algorithm for graphs with separators of size $O(\log \log n)$ (satisfied in most practical applications), having a near-exact extension for more general bounded genus graphs. It provides favorable near-linear matrix-vector computation time with kernel matrices, particularly relevant for processing massive input graphs. The latter gives additional computational gains (at the cost of the accuracy loss), even for medium-size input graphs and is particularly straightforward to implement, as it does not involve any combinatorial operations.


\section{Preliminaries}
\label{sec:methodology}

\subsection{Problem setting}
\label{sec:background}
Consider two metric spaces: $\mathcal{X}$ and $\mathcal{Y}$, with the corresponding probabilistic measures $\mu$ and $\nu$ defined on them. We will assume that $\mathcal{X}$ and $\mathcal{Y}$ are discretized, with the corresponding sets of points $X=\{x_i\}_{i=1}^m$ and $Y=\{y_j\}_{j=1}^n$ respectively for some $m,n>0$. We encode the discretized versions of $\mu$ and $\nu$ as: $\mathbf{a}\in\Delta_m$ and $\mathbf{b}\in\Delta_n$ ($\Delta_{k}\overset{\mathrm{def}}{=}\{(z_{1},...,z_{k}):z_{i} \geq 0; \sum_{i=1}^{k} z_{i}=1\}$).
The internal discretized geometries of $\mathcal{X}$ and $\mathcal{Y}$ are represented by cost matrices $\mathbf{C}_X\in\R^{n\times n}$ and $\mathbf{C}_Y\in\R^{m\times m}$. 
In this paper, we are interested in the setting, where
$C_{Z}(i,j) \overset{\mathrm{def}}{=} d_{G_{Z}}(i,j)$ ($Z \in  \{X,Y\}$) for $d_{G}(i,j)$ denoting the length of the shortest path distance between vertices $i$ and $j$ in graph $G$. Here $G_{Z}$ is a graph with the set of vertices $Z$ and weights of edges defined as distances in the underlying metric spaces.
A coupling $\mathbf{T}\in\R_{\geq 0}^{m\times n}$ is a soft, generally many-to-many correspondence between the two distributions $\mathbf{a}$ and $\mathbf{b}$, i.e.:
\begin{equation}
\sum_{j=1}^{n} T_{i,j}=a_{i}, \textrm{ } \sum_{i=1}^{m} T_{i,j}=b_{j}    
\end{equation}


\subsection{Gromov-Wasserstein Distance, entropic GWD and fused GWD (FGWD)}
\label{sec:gw-fgw}
For a given coupling $\mathbf{T} \in \R_{\geq 0}^{m\times n}$, consider the following loss function:
\begin{equation}
\mathcal{G}(\mathbf{T})=\frac{1}{2}
\sum_{i,k=1}^{m}\sum_{j,\ell=1}^{n}
\left(C_X(i,k)-C_Y(j,\ell)\right)^2T_{ij}T_{k\ell}.
\label{eq:gw-half-loss}
\end{equation}
Intuitively, $\mathcal{G}(\mathbf{T})$ measures how well the coupling conducts a matching between $X$ and $Y$, by looking into the discrepancy between intra-metric points' relationships (given as costs). Quadratic loss is a natural choice. Importantly the matching is soft, thus we do not require $m=n$. Probabilities at points can be interpreted as weights, providing a convenient way to measure uncertainties about the data (smaller weights imply large uncertainties). Since costs do not involve cross-metric terms, the loss defined in Eq. \ref{eq:gw-half-loss} is valid also on metric spaces that do not relate to each other in any way.
As previously noticed, the admissible couplings satisfy the following:
\begin{equation}
\Pi(\mathbf{a},\mathbf{b})= \left\{\mathbf{T}\geq0:
\mathbf{T}\one_n=\mathbf{a},\quad \mathbf{T}^\top\one_m=\mathbf{b}\right\}.
\end{equation}
The \textit{Gromov-Wasserstein Distance} (GWD) is then defined as follows:
\begin{equation}
\mathrm{GWD}(\mathbf{a},\mathbf{b}) = \min_{\mathbf{T} \in \Pi(\mathbf{a},\mathbf{b})} \mathcal{G}(\mathbf{T}).
\end{equation}
The entropic relaxation of the above definition leverages an entropy-based regularizer:
\begin{equation}
\mathrm{GWD}_{\epsilon}(\mathbf{a},\mathbf{b}) = 
\min_{\mathbf{T}\in\Pi(\mathbf{a},\mathbf{b})} \left[\mathcal{G}(\mathbf{T}) +\varepsilon\sum_{i,j}T_{ij}(\log T_{ij}-1)\right].
\label{eq:entropic-gw}
\end{equation}
One of the standard approaches to solving entropic GWD is by the 2-level iterative procedure. Here the outer loop iterates over different transport matrices $\mathbf{T}$ and the inner loop conducts regular Sinkhorn iterations, as described in Sec. \ref{sec:introduction}, but for the kernel matrix $\mathbf{K}_{\epsilon}$ defined as: $\mathbf{K}_{\epsilon}=\exp(-\frac{\mathbf{D}_{\mathbf{T}}}{\epsilon})$, where $\mathbf{D}_{\mathbf{T}}$ is given as follows for $\mathbf{p}=\mathbf{T}\mathbf{1}_{n}$ and $\mathbf{q}=\mathbf{T}^{\top}\mathbf{1}_{m}$:
$\mathbf{D}_{\mathbf{T}} \in \mathbb{R}^{n \times m}$:
\begin{equation}
\mathbf{D}_\mathbf{T}= (\mathbf{C}_X^{\circ2}\mathbf{p})\one_m^\top
+\one_n(\mathbf{C}_Y^{\circ2}\mathbf{q})^\top
-2\mathbf{C}_X\mathbf{T}\mathbf{C}_Y^\top,
\label{eq:gw-linearization}
\end{equation}
Here $\circ$ stands for the Hadamard (element-wise) product. After each step of the inner loop, new transport matrix $\mathbf{T}$ is formed and next outer loop step is invoked.

Even though, as emphasized earlier, regular GWD does not introduce any cross-cost terms, there exists its generalization that leverages the so-called \textit{cross-cost} matrix $\mathbf{M} \in \mathbb{R}^{m \times n}$. Matrix $\mathbf{M}$ is usually given as: $M_{ij}=\|\mathbf{f}_i-\mathbf{g}_j\|_2^2$ for two sets of vectors $\{\mathbf{f}_{i}\}_{i=1}^{m}$ and $\{\mathbf{g}_{j}\}_{j=1}^{n}$.

That generalization effectively replaces the $\mathcal{G}(\mathbf{T})$ objective with the $\alpha \mathcal{G}(\mathbf{T})+(1-\alpha)\mathbf{M} \circ \mathbf{T}$ (for a hyperparameter $\alpha \in [0,1]$), and we refer to it as \textit{Fused Gromov-Wasserstein Distance} (FGWD). Its entropic relaxation is given by the following formula:
\begin{equation}
\mathrm{FGWD}_{\epsilon,\alpha}(\mathbf{a},\mathbf{b})=\min_{\mathbf{T}\in\Pi(\mathbf{a},\mathbf{b})}
\left[\alpha\mathcal{G}(\mathbf{T}) +(1-\alpha)\mathbf{M} \circ \mathbf{T}
+\varepsilon\sum_{i,j}T_{ij}(\log T_{ij}-1)\right].
\label{eq:fgw}
\end{equation}
The corresponding kernel matrix $\mathbf{K}_{\epsilon,\alpha}$ is now given as
$\mathbf{K}_{\epsilon,\alpha} = \exp(-\frac{\mathbf{Q}_{\mathbf{T}}}{\epsilon})$
for:
\begin{equation}
\label{eq:qmatrix}
    \mathbf{Q}_\mathbf{T}=\alpha \mathbf{D}_\mathbf{T}+(1-\alpha)\mathbf{M}.
\end{equation}
We normalize source/target structural costs by a shared, template-derived scale. Independently normalizing two spaces would erase global scale differences and define a different matching problem.

\vspace{-2mm}
\section{Efficient Geodesic Gromov-Wasserstein (EGGroW)}
\label{sec:eg-grow}
We are ready to present the class of the Efficient Geodesic Gromov-Wasserstein (EGGroW) algorithms. We focus on the entropic approach and start with exact algorithms that rely on the recently introduced \textit{GenusSink} method \citep{genussink}. The only assumption we need to make regarding that exact class is that graphs under consideration have separators of size $O(\log \log (n))$ (near -exact extension exists, as shown in \citep{genussink}, for even more general bounded genus graphs). We then propose even more computationally efficient variants, but at the cost of approximating the original solution. We will focus on FGWD, as the most general variant.
\vspace{-2mm}
\subsection{Exact GWD Computations with Deterministic EGGROW}
\label{sec:separator-moments}
Looking at Eq. \ref{eq:gw-linearization}, \ref{eq:qmatrix}, we conclude that the main computational bottlenecks present while computing entropic FGWD involve: multiplications with matrices $\mathbf{C}$ and $\mathbf{C}^{\circ2}$, Hadamard products with matrices $\mathbf{M}$ and element-wise exponentiations.
Rather than storing $\mathbf{C}$ and $\mathbf{C}^{\circ2}$ explicitly, we will instead apply the efficient \textbf{deterministic} polylog-linear GenusSink algorithm \citep{genussink} for matrix-vector multiplications involving $\mathbf{C}$. GenusSink leverages separator-based factorization of low genus graphs \citep{gilbert1984separator,choromanski2023graphfields} for efficient computations involving a broad range of cost matrices relying on distances between graphs' vertices (defined via shortest paths).

Our first observation is that a small modification of the GenusSink algorithm can be applied to efficient multiplications with both $\mathbf{C}$ and $\mathbf{C}^{\circ2}$ (in \citep{genussink} the algorithm is applied to efficient multiplications with $\exp(\mathbf{C})$). We prove it in the lemma below.
\begin{lemma}[Efficient multiplications with $\mathbf{C}$ and $\mathbf{C}^{\circ2}$]
For $n$-vertex graphs under consideration with separators of size $O(\log \log n)$, multiplications with matrices $\mathbf{C}$ and $\mathbf{C}^{\circ2}$, where $C_{i,j}=d(i,j)$ is the shortest path distance between vertices $i$ and $j$, can be conducted in polylog-linear time.      
\end{lemma}
\begin{proof}
We will leverage time complexity analysis of the cross\_compute method from \citep{genussink}. By leveraging it, we deduce that it suffices to prove that for any two sequences of scalars $(x_{1},...,x_{m})$,$(y_{1},...,y_{n})$, matrix-vector multiplications with two matrices: $\mathbf{P}=[x_{i}+y_{j}]_{i=1,...,m}^{j=1,...,n}$ and $\mathbf{P}^{\circ 2}$ can be conducted in polylog-linear time. Note that for for any $\mathbf{v} \in \mathbb{R}^{n}$, we have: $(\mathbf{Pv})_{i}=x_{i}\sum_{j=1}^{n}v_{j}+\sum_{j=1}^{n}y_{j}v_{j}$ for $i=1,...,m$. Thus we can compute all $(\mathbf{Pv})_{i}$ in time $O(m+n)$.
Now let us focus on $\mathbf{P}^{\circ 2}$. We have: $\mathbf{P}^{\circ 2}=[x_{i}^{2}+2x_iy_j+y_j^{2}]_{i=1,...,m}^{j=1,...,n}$. The $O(m+n)$ time matrix-vector multiplications with $\mathbf{P}^{\circ 2}$ follows naturally from our analysis for $\mathbf{P}$ and the fact that matrix $[x_{i}y_{j}]_{i=1,...,m}^{j=1,...,n}$ trivially support $O(m+n)$ matrix-vector multiplication via its rank-1 decomposition.
\end{proof}
For given matrices $\mathbf{\Lambda}$, $\mathbf{V}$ ($\mathbf{V}$ potentially being a vector), define the action of $\mathbf{\Lambda}$ on $\mathbf{V}$ as: $\mathcal{S}_{\mathbf{\Lambda}}(\mathbf{V})$. We introduce extra notation rather than denoting this action simply as $\mathbf{\Lambda V}$ to highlight that its computation might not require brute-force matrix-matrix multiplication. With that notation in place, we can re-write matrix $\mathbf{D}_{\mathbf{T}}$ from Equation \ref{eq:gw-linearization} as follows, using intermediate calculations:
\begin{align}
    \mathbf{A}&=\mathcal{S}_{\mathbf{C}_X^{\circ2}}(\mathbf{p}),
    &\mathbf{B}&=\mathcal{S}_{\mathbf{C}_Y^{\circ2}}(\mathbf{q}),\\
    \mathbf{U}&=\mathcal{S}_{\mathbf{C}_X}(\mathbf{T}),
    &\mathbf{V}&=\mathcal{S}_{\mathbf{C}_Y}(\mathbf{U}^\top)^\top,\\
    \mathbf{D}_{\mathbf{T}}&=\mathbf{A}\one_m^\top+\one_n\mathbf{B}^\top-2\mathbf{V}.
\end{align}
We obtained exact entropic GWD algorithm, where cubic computations involving the construction of a matrix $\mathbf{D}_{\mathbf{T}}$ from Equation \ref{eq:gw-linearization} (namely, the computation of $\mathbf{C}_{X}\mathbf{T}\mathbf{C}_{Y}^{\top}$) were replaced by \textbf{polylog-quadratic}. We refer to this method as a \textbf{deterministic EGGroW}, or simply: \textbf{DET-EGGroW}.


\subsection{Randomized Approximate EGGroW for Additional Speedups}
\label{sec:prf}
We will now show how the additional speedups for the EGGroW algorithm can be obtained by applying the theory of random features \citep{rahimi2007random, linear-time-sinkhorn-positive-features}. We will use the following notation.
\begin{equation}
    \mathbf{U}=\mathbf{C}_X\mathbf{T},
    \quad
    \mathbf{A}=\mathbf{C}_X^{\circ2}\mathbf{p},
    \quad
    \mathbf{B}=\mathbf{C}_Y^{\circ2}\mathbf{q},
\end{equation}
and define \textit{augmented query and key vectors} $\{\mathbf{z}_{i}\}_{i=1}^{m}$, $\{\mathbf{w}_{j}\}_{j=1}^{n}$  as follows, where $M_{i,j}=\|\mathbf{f}_{i}-\mathbf{g}_{j}\|^{2}_{2}$:
\begin{equation}
\mathbf{z}_i=
\begin{bmatrix}
\sqrt{2\alpha/\varepsilon}\,\mathbf{U}_{i,:}\\
\sqrt{2(1-\alpha)/\varepsilon}\,\mathbf{f}_i
\end{bmatrix},
\qquad
\mathbf{w}_j=
\begin{bmatrix}
\sqrt{2\alpha/\varepsilon}\,(\mathbf{C}_Y)_{j,:}\\
\sqrt{2(1-\alpha)/\varepsilon}\,\mathbf{g}_j
\end{bmatrix}.
\end{equation}
Each iteration of the entropic FGWD computation critically relies on the multiplications with kernel matrices $\mathbf{K}_{\mathbf{T}}=[\exp(-Q_T(i,j)/\varepsilon)]_{i=1,...,n}^{j=1,...,m}$. Our key observation is that those matrices admit \textbf{unbiased} low-rank factorization. To see that, note first that $\mathbf{K}_{\mathbf{T}}$ are related to softmax-kernel matrices: 
\begin{equation}
\mathbf{K}_{\mathbf{T}}(i,j)=\exp(-Q_T(i,j)/\varepsilon)
=s_i\exp(\mathbf{z}_i^\top \mathbf{w}_j)t_j,
\end{equation}
where
\begin{align}
s_i&=\exp\!\left(-\frac{\alpha A_i+(1-\alpha)\|\mathbf{f}_i\|_2^2}{\varepsilon}\right),\\
t_j&=\exp\!\left(-\frac{\alpha B_j+(1-\alpha)\|\mathbf{g}_j\|_2^2}{\varepsilon}\right).
\end{align}
Using positive feature maps $\phi:\mathbb{R}^{n} \rightarrow \mathbb{R}^{r}$ from \citep{performers}, for the unbiased estimation of the softmax-kernel, we approximate it as follows for 
$\omega_{1},...\omega_{r} \overset{\mathrm{iid}}{\sim} \mathcal{N} (0,\mathbf{I}_{r})$:
\begin{align}
\begin{split}
    \exp(\mathbf{z}_i^\top \mathbf{w}_j) \overset{\mathbb{E}}{=} 
\phi(\mathbf{z}_i)^\top\phi(\mathbf{w}_j), \\ 
    \phi(\mathbf{x}) \overset{\mathrm{def}}{=}\frac{1}{\sqrt{r}}
    \exp\left(-\frac{\|\mathbf{x}\|^{2}}{2}\right) \left(\exp(\omega^{\top}_{1}\mathbf{x}),...,\exp(\omega^{\top}_{r}\mathbf{x})\right)^{\top}.
\end{split}
\end{align}
That directly leads to the following unbiased low-rank factorization of the kernel matrix $\mathbf{K}_{\mathbf{T}}$.
\begin{equation}
\mathbf{K}_\mathbf{T} \overset{\mathbb{E}}{=}
(\diag(\mathbf{s})\Phi^{\top}_X)(\Phi_Y\diag(\mathbf{t})),
\label{eq:prf-kernel}
\end{equation}
where $\Phi_{X} \in \mathbb{R}^{r \times m}$ and $\Phi_{Y} \in \mathbb{R}^{r \times n}$ are obtained by stacking (as columns) vectors $\{\phi(\mathbf{z}_{i})\}_{i=1}^{n}$ and $\{\phi(\mathbf{w}_{j})\}_{j=1}^{m}$ respectively. 
Sinkhorn scaling preserves this factorization, so the approximate coupling can be stored as $\mathbf{T}=\mathbf{L}^{\top}\mathbf{R}$ for low-rank matrices $\mathbf{L},\mathbf{R}$. 

In practice, centering and diagonal balancing are used to reduce the variance of the approximation. All scale factors must be retained in semi-relaxed problems because, unlike balanced Sinkhorn, separable factors can affect a learned marginal.  
Finally, note that vectors $\mathbf{A}$ and $\mathbf{B}$ can be computed in polylog-linear time, with the use of the techniques developed for DET-EGGroW (see: Sec. \ref{sec:separator-moments}) and furthermore, $\mathbf{U}$ can be computed in either: polyloq-quadratic time (again, using techniques from Sec. \ref{sec:separator-moments}) or in time $O(m^{2}r)$ (where $r$ is the number of random features) if low-rank decomposition of $\mathbf{T}$ is leveraged.
Thus the resulting algorithm also runs in \textbf{near-quadratic time}, rather than cubic, as its brute-force counterpart. We refer to it as an \textbf{approximate EGGroW}, or simply: \textbf{APP-EGGroW}. The computational advantage of APP-EGGroW over DET-EGGroW is that in practice operations involving random features are better suited for modern accelerators than separator-based approaches. The price is some accuracy loss. We analyze this in more depth in Sec. \ref{sec:exp}.
Table \ref{tab:representations} summarizes the computational profiles of the brute-force algorithm, as well as two EGGroW variants.



\begin{table}[t]
\centering
\small
\caption{The three implementations of the Gromov-Wasserstein Distance calculation remove different dense objects. Here $r$ is the number of random features for APP-EGGroW. Operator storage depends on the selected separator or analytic moment backend.}
\label{tab:representations}
\begin{tabular}{llll}
\hline
Method & Intra-space distances & Cross-space state & Status\\
\hline
Dense GWD (brute-force) & Dense $\mathbf{C}_X,\mathbf{C}_Y$ & Dense $\mathbf{T},\mathbf{D}_\mathbf{T},\mathbf{K}_\mathbf{T}$ & Exact\\
DET-EGGroW & Operator form & Dense $\mathbf{T},\mathbf{D}_\mathbf{T},\mathbf{K}_\mathbf{T}$ & Exact if operators exact\\
APP-EGGroW & Operator form & $O((n+m)r)$ factors & Approximate\\
\hline
\end{tabular}
\end{table}


\section{EGGroW for 3D Template Detection}
\label{sec:template-method}

In the previous section, we described in depth the class of EGGroW algorithms. We will now propose a method that applies Fused Gromov-Wasserstein Distance for the 3D template detection problem. EGGroW algorithms make this method practical from the computational standpoint.

\subsection{Observation model and subsampling}
\label{sec:observation}

Our input is a point cloud or a mesh (e.g. CAD) template
$\mathbf{Y}=\{\mathbf{y}_j\}_{j=1}^{n}\subset\R^3$ and a calibrated multi-view RGB-D observation of a scene.  The solver-visible data comprise depth returns, camera calibration, template geometry, and any descriptors computed from those observations.  Instance labels, the generated object pose, and visibility truth are held in a separate evaluation object.

We voxel-subsample the observed scene without using instance labels and sample or aggregate the template surface into $m$ discrete nodes $\mathbf{X}=\{\mathbf{x}_{1},...,\mathbf{x}_{m}\}$.  Template nodes carry a normalized surface-area prior $\mathbf{b}\in\Delta_n$; the subsampled scene carries prior $\mathbf{a}\in\Delta_m$.  We report the raw RGB-D count, retained scene count, voxel size, template vertex count, terminal count, and surface-area ownership construction. 

\subsection{Stage I: capacity-constrained scene localization}
\label{sec:stage1}

Let $\rho_s\in(0,1]$ be the minimum scene coverage measured with respect to the scene prior $\mathbf{a}$.  We solve:
\begin{equation}
\begin{aligned}
\min_{\mathbf{T}\geq0}\quad&
\alpha\mathcal{G}(\mathbf{T})
+(1-\alpha) \mathbf{M} \circ\mathbf{T}
+\varepsilon\sum_{i,j}T_{ij}(\log T_{ij}-1)
\quad \text{s.t.}\quad
\mathbf{T}^\top\one_m=\mathbf{b},
\mathbf{T}\one_n\leq \frac{\mathbf{a}}{\rho_s}.
\end{aligned}
\label{eq:stage1}
\end{equation}
The first constraint fixes the complete template marginal to $\mathbf{b}$, so the coupling has total mass one, while the scene distribution $\mathbf{p}=\mathbf{T}\one_n$ is learned. The second constraint limits the mass assigned to each scene point to $p_i\leq a_i/\rho_s$. Consequently, the coupling cannot concentrate all its mass on a subset of scene points whose total prior mass is less than $\rho_s$. To score scene relevance, we compare the learned mass at each point with its prior weight: $s_i^{\mathrm{scene}}=p_i/a_i.$ This multiplicative change accounts for nonuniform prior weights, rather than ranking points by raw correspondence mass $p_i$. A threshold chosen on calibration data can then define the predicted scene support for Stage II; no threshold may be chosen from test labels. 

\subsection{Stage II: template visibility and correspondence}
\label{sec:stage2}

Stage II fixes the predicted scene distribution $\mathbf{\widehat p}$ and learns the template marginal \citep{vincentcuaz2022semirelaxed}:
\begin{equation}
    \mathbf{T}\one_n=\mathbf{\widehat p},
    \qquad
    \mathbf{q}=\mathbf{T}^\top\one_m.
\end{equation}
The value $q_j$ is the recovered correspondence mass assigned to node
$j$.  Two forms of partial-template control are useful:
\begin{enumerate}
    \item a capacity $q\leq b/\rho_t$, where $\rho_t$ is minimum template
    coverage; or
    \item a soft marginal penalty $\tau\KL(q\|r)$ around a strictly positive,
    truth-free reference $r$.
\end{enumerate}
Since the template's position and orientation are initially unknown, we estimate an initial rigid alignment $(\mathbf{R}_0,\mathbf{t}_0)$ from the correspondences obtained in Stage I, discussed in App. \ref{app:stage1-pose}. This alignment rotates and translates the template into scene coordinates without changing its shape, and is held fixed throughout Stage II. Using the calibrated camera views, we raycast the aligned template to obtain the visibility reference $\mathbf{q}_{\mathrm{ray}}$. In Sec.~\ref{sec:exp-template}, we set $\mathbf{r}=0.98\mathbf{q}_{\mathrm{ray}}+0.02\mathbf{b}$ to retain positive support everywhere. 

Let $D$ be the template diameter and $G_X,G_Y$ the scene and template graphs, respectively. Writing $d_G(u,v)$ for the shortest-path distance between vertices $u$ and $v$ on graph $G$, we define the costs
\begin{align}
C_X(i,k) = d_{G_X}(\mathbf{x}_i,\mathbf{x}_k)/D,
\qquad
C_Y(j,\ell) = d_{G_Y}(\mathbf{y}_j,\mathbf{y}_\ell)/D.
\label{eq:template-shortest-path-cost}
\end{align}
To also favor matches consistent with the upstream pose, we define a cross-cost that penalizes the squared Euclidean distance between a scene point and the predicted location of a template point:
\begin{equation}
M^{\mathrm{pose}}_{ij} =\frac{\|\mathbf{x}_i-
(\mathbf{y}_j\mathbf{R}_0^\top+\mathbf{t}_0)\|_2^2}{D^2}.
\end{equation}
Division by $D^2$ makes this quantity dimensionless.
Stage II then solves the following entropically regularized $\textrm{FGWD}_{\epsilon,\alpha}$ objective with an additional KL penalty, using the EGGroW algorithm:
\begin{equation}
\begin{aligned}
\min_{\mathbf{T}\geq0}\quad&
\lambda_G\mathcal{G}(\mathbf{T})
+\lambda_M\langle \mathbf{M}^{\mathrm{pose}},\mathbf{T}\rangle+\varepsilon\sum_{i,j}T_{ij}(\log T_{ij}-1)
+\tau\KL(\mathbf{q}\|\mathbf{r})\\
\text{s.t.}\quad&\mathbf{T}\one_n=\mathbf{\widehat p},
\qquad \mathbf{q}=\mathbf{T}^\top\one_m
\end{aligned}
\label{eq:stage2}
\end{equation}
The KL term is a soft visibility penalty: $\mathbf{q}$ can depart from the raycast reference when the relational geometry favors another visible region. Additional details for how we integrate the KL term with the FGWD objective in the optimization formulation are reported in App. \ref{app:kl-sinkhorn}. In contrast, the row-marginal constraint is enforced exactly; any capacity or affine-moment constraints used in a variant are likewise imposed as hard constraints. As we show in Sec. \ref{sec:exp-template}, EGGroW algorithm makes this 3D template detection algorithm computationally feasible for large-size templates and scenes. We will also show in Sec. \ref{sec:exp-template} that proposed EGGroW-powered 3D detection method is capable of finding templates of various profiles, in particular of irregular shapes.

\section{Experiments}
\label{sec:exp}
\subsection{Pose Correspondence}
\label{sec:exp-pose}

We evaluate EGGroW on FAUST meshes \cite{Bogo:CVPR:2014} (see Fig. \ref{fig:faust-gw-correspondence}). The registrations share vertex identities, giving ground-truth correspondences. We use source scan $004$ and three target scans from the same subject: $007$, $006$, and $003$. For each pair, we sample $n=512,1024,2000,4096,6890$ vertices, construct weighted mesh graphs from triangle edges, compute sampled geodesic matrices through the full graph, and use uniform marginals. For FGWD, $M_{ij}$ is the squared Euclidean distance between normalized 3D coordinates and we set $\alpha=0.95$, so the objective is mostly intrinsic GWD with a small Euclidean feature term. 

\begin{figure}[!ht]
\centering
\includegraphics[width=\linewidth]{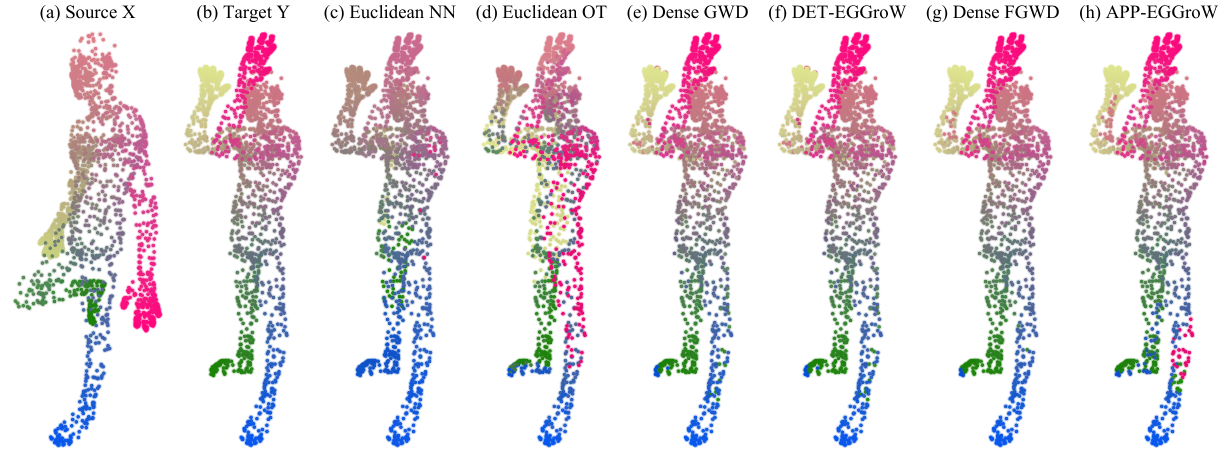}
\caption{FAUST pose correspondence for (a) source scan $004$ and (b) target scan $003$. Colors are intrinsic source colors transferred to the target by each method (c)-(h).}
\label{fig:faust-gw-correspondence}
\end{figure}

We compare against two Euclidean baselines. Euclidean nearest neighbors (NN) independently maps each source point to the closest target point in normalized 3D space: $j(i)=\arg\min_j \|\mathbf{x}_i-\mathbf{y}_j\|^2.$ This baseline is purely local and can send many source points to the same target region. Euclidean OT instead solves entropic optimal transport with cost $M_{ij}=\|\mathbf{x}_i-\mathbf{y}_j\|^2$ and uniform marginals. This enforces global mass balance, but it is still based only on extrinsic 3D proximity. Both baselines can therefore fail when unrelated body parts become close in Euclidean space. Tab.~\ref{tab:faust-gw-results} shows performance, which we primarily measure using mean target-geodesic correspondence error; lower is better. Further details and results can be found in App.~\ref{sec:additional-numerical-results}.

\begin{table}[!ht]
\centering
\small
\caption{Mean target-geodesic correspondence error on three FAUST same-subject pose pairs at $n=2000$ vertices. Euclidean OT and FGWD use squared Euclidean cross-feature costs.}
\setlength{\tabcolsep}{4pt}

\begin{tabular}{lrrrrrrrr}
\hline
Pair & NN & OT & GWD & FGWD & DET
& APP$_{64}$ & APP$_{256}$ & APP$_{1024}$ \\
\hline
\multicolumn{9}{l}{\textbf{Geodesic correspondence error}
$\downarrow$} \\[2pt]
$004\to003$ & 0.574 & 0.636 & \textbf{0.332}
& \textbf{0.166} & \textbf{0.166}
& 0.240 & 0.232 & 0.230 \\
$004\to006$ & 0.518 & 0.641 & \textbf{0.323}
& \textbf{0.108} & \textbf{0.108}
& 0.252 & 0.206 & 0.206 \\
$004\to007$ & 0.665 & 0.633 & \textbf{0.354}
& \textbf{0.223} & \textbf{0.223}
& 0.281 & 0.261 & 0.237 \\
\hline
\multicolumn{9}{l}{\textbf{Runtime (seconds)}
$\downarrow$} \\[2pt]
$004\to003$ & - & - & $896.4$
& 1077.2 & 587.0 & \textbf{320.8} & 357.1 & 535.1 \\
$004\to006$ & - & - & $794.3$
& 878.8 & 585.3 & \textbf{322.3} & 354.4 & 539.0 \\
$004\to007$ & - & - & $831.1$
& 996.5 & 585.0 & \textbf{322.7} & 357.1 & 531.3 \\
\hline
\end{tabular}

\label{tab:faust-gw-results}
\end{table}

Fig.~\ref{fig:faust-prf-fgw-timing} shows the expected accuracy--runtime tradeoff for APP-EGGroW. Increasing the number of random features improves the mean geodesic error, with $1024$ features giving the best APP-EGGroW accuracy in this pose pair.  At the same time, smaller feature counts are faster than the Dense GWD/FGWD baselines, and both DET-EGGroW and APP-EGGroW runtimes grow significantly slower than baseline methods.

\begin{figure}[!ht]
    \centering
    \includegraphics[width=\linewidth]{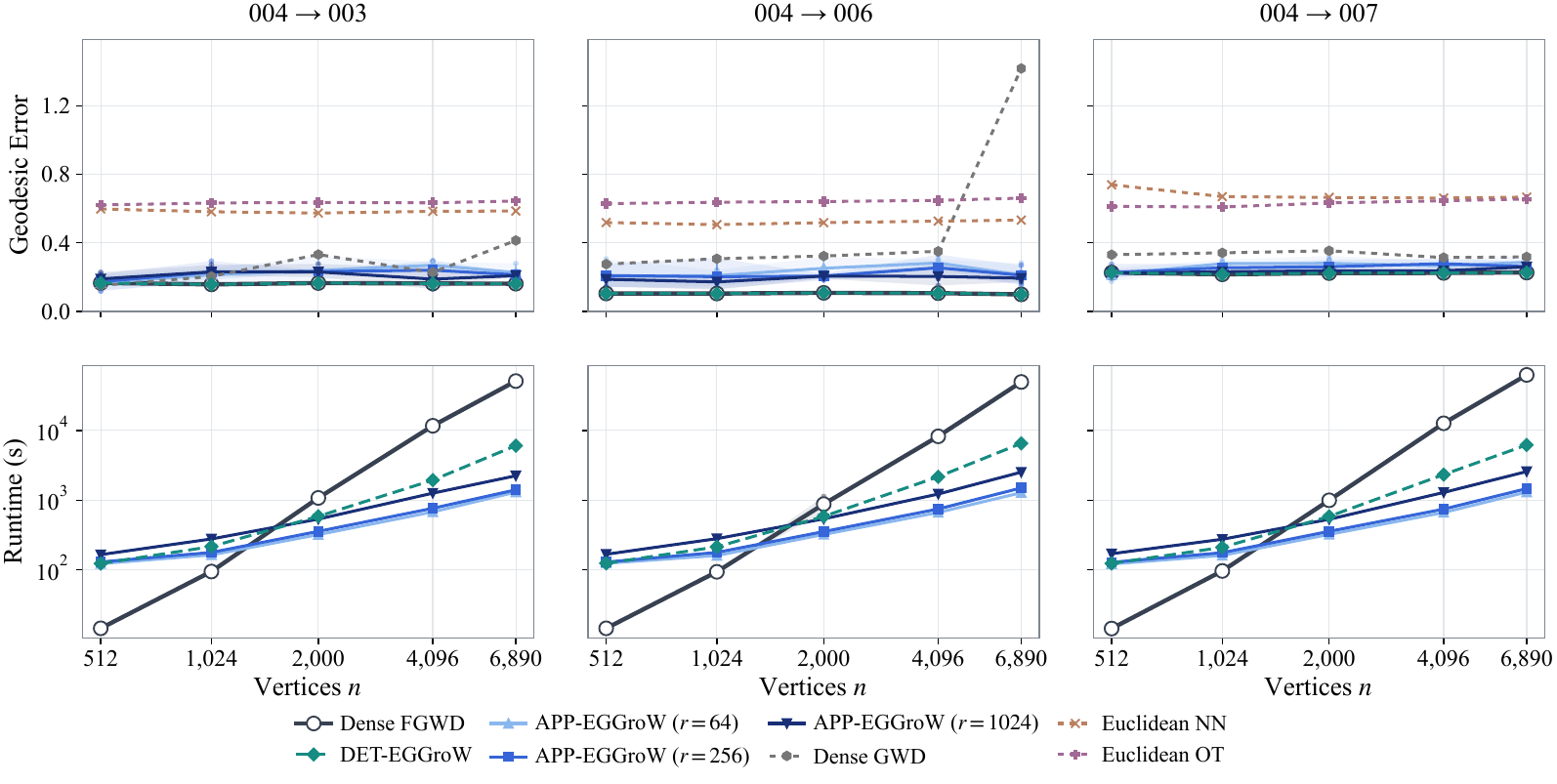}
    \caption{Error (top) and runtime (bottom) across different methods as $n$ increases.}
    \label{fig:faust-prf-fgw-timing}
\end{figure}

\subsection{Template Detection}
\label{sec:exp-template}
We compare Dense GWD, Dense FGWD, DET-EGGroW, APP-EGGroW, and Scetborn-Cuturi \cite{LT-GWD} across increasing sizes meshes—Fig.~\ref{fig:octopus-room-template-scene} shows the scene (i.e., a room) and one of the templates (i.e., the octopus tail). ET-EGGroW matches dense transport plans to numerical precision and accelerates larger ribbon and point-cloud instances, but remains slower on sphere–ellipsoid meshes because of preprocessing and iteration overhead. APP-EGGroW is fastest on the larger Lego block and Octopus instances, at the cost of transport-plan approximation error. For more setup details, see Sec.~\ref{sec:template-method} and App.~\ref{sec:template-detection-appendix}.

\begin{figure}[!ht]
    \centering
    \includegraphics[width=\linewidth, trim={0 0 8mm 0}, clip]{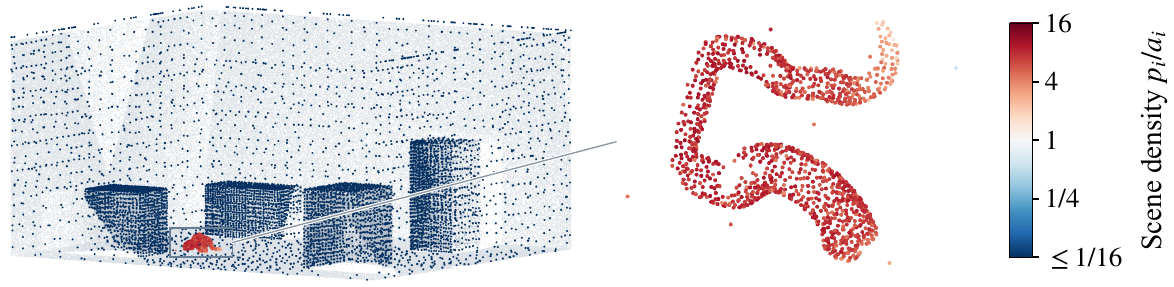}
    \caption{Octopus template localization where colors show the recovered scene density $p_i/a_i$; gray points provide scene context.}
    \label{fig:octopus-room-template-scene}
\end{figure}

\begin{figure}[!ht]
    \centering
    \includegraphics[width=1\linewidth]{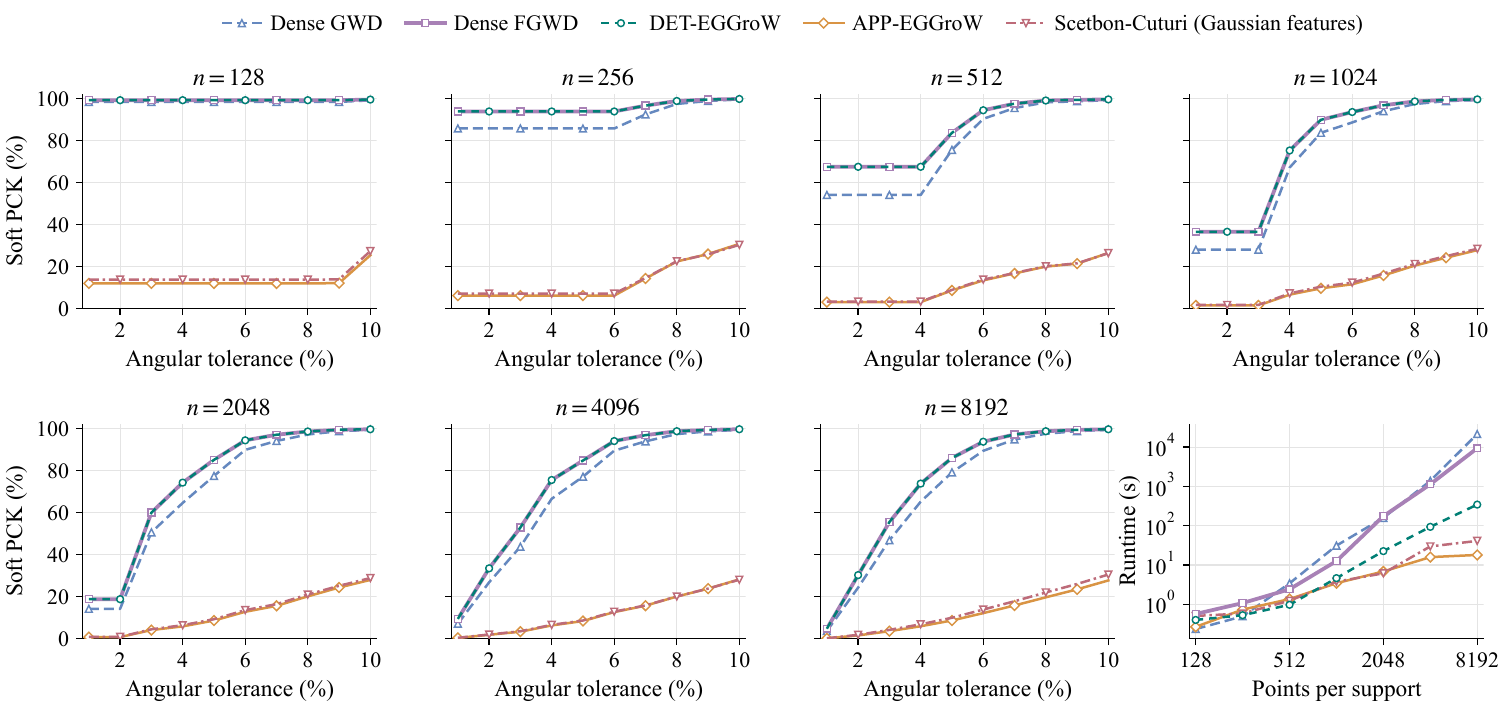}
    \caption{Mean accuracy and runtime (bottom right) at increasing mesh sizes for each of the different methods under comparison.}
    \label{fig:placeholder}
\end{figure}

A successful transport plan concentrates mass around the correct correspondences. We measure this concentration using soft PCK, that is, the fraction of transport mass assigned to pairs within a prescribed error tolerance. For our surfaces, correspondence error is the angle between their known underlying spherical directions. A 10\% tolerance therefore permits an angular error of 18°, capturing how reliably each method recovers the correct region of the shape. Our results show that DET-EGGroW follows Dense FGWD across the entire tolerance range, preserving the reference method's correspondence accuracy. Feature fusion is particularly helpful when matches must be precise: at $n=1024$ and $1\%$ tolerance, soft PCK increases from approximately $28\%$ for Dense GWD to $37\%$ for Dense FGWD and DET-EGGroW. Furthermore, at $10\%$ tolerance, all three exceed $99\%$, indicating that most transport mass remains near the correct region of the shape. Finally, limitations and future work for both template detection and pose correspondence can be found in App.~\ref{sec:limitations-future-work}.
\section{Conclusion}
\label{sec:conclusion}
We presented in this paper \textbf{E}fficient \textbf{G}eodesic \textbf{Gro}mov-\textbf{W}asserstein (EGGroW) algorithms for efficiently computing Gromov-Wasserstein distances, with metrics induced by shortest-path distance on graphs. EGGroW is characterized by near-quadratic computational time (as compared to the cubic of its brute-force entropic GWD counterpart) and leverages methods based on graphs factorization with small-size separators, proposed in GenusSink algorithm, as well as positive random features techniques used in efficient attention modules for Transformers. Our theoretical analysis is confirmed by experimental validation on the problems on 3D pose estimation and 3D template detection. Our proposed GWD-based template detection algorithm critically relies on the computational efficiency of EGGroW to handle large input point clouds / meshes.

\subsection*{AI use statement}

We have not used generative AI tools to produce the content of this paper.

\subsection*{Ethics statement}

We believe that our paper does not raise any questions regarding the Code
of Ethics. In particular, none of our studies involves human subjects. We are also not aware of any potential conflicts of interest related to this submission.

\subsection*{Reproducibility statement}

We have described in depth all presented algorithms, as well as the experimental setup in which they were used, so that all our experimental results can be accurately reproduced. In particular, we dedicate a separate section just to describe in depth proposed 3D template detection algorithm (see: Sec. \ref{sec:template-method}). All additional details (hyperparams, datasets used, etc.) are listed in the corresponding sections (see in particular: Sec. \ref{sec:exp}). We will open source our code upon acceptance of the manuscript.

\bibliography{iclr2027_conference}
\bibliographystyle{iclr2027_conference}

\appendix
\section{Additional details for 3D template detection}

\subsection{Initial rigid alignment from the Stage I coupling}
\label{app:stage1-pose}

Stage I provides a soft correspondence matrix $\mathbf T^{(1)}$, where $T^{(1)}_{ij}$ measures the correspondence mass between scene point $\mathbf x_i$ and template point $\mathbf y_j$. We use these correspondences to estimate a single rigid transformation of the template into scene coordinates. With point coordinates written as row vectors, the transformation maps
\[
\mathbf y_j \mapsto \mathbf y_j\mathbf R^\top+\mathbf t.
\]
We estimate the initial rotation and translation by solving
\begin{equation}
(\mathbf R_0,\mathbf t_0)
\in \operatorname*{arg\,min}_{
\mathbf R\in\mathrm{SO}(3),\,
\mathbf t\in\mathbb R^{1\times3}}
\sum_{i=1}^{m}\sum_{j=1}^{n}
T^{(1)}_{ij} \left\| \mathbf x_i-
(\mathbf y_j\mathbf R^\top+\mathbf t) \right\|_2^2,
\label{eq:stage1-pose-fit}
\end{equation}
where
\[
\mathrm{SO}(3) = \left\{\mathbf R\in\mathbb R^{3\times3}:
\mathbf R^\top\mathbf R=\mathbf I_3,\, \det(\mathbf R)=1
\right\}.
\]
Thus, scene--template pairs with larger Stage I correspondence mass have greater influence on the fitted pose. We use all weighted candidate matches rather than first selecting a single match for each point.

\paragraph{Weighted centering.}
Let
\[
\mathbf p^{(1)}=\mathbf T^{(1)}\one_n
\]
be the scene marginal learned in Stage I. Since the template marginal is fixed to $\mathbf b$, the weighted centroids of the scene and template are
\begin{equation}
\bar{\mathbf x} =
\sum_{i=1}^{m}p_i^{(1)}\mathbf x_i,
\qquad \bar{\mathbf y}=\sum_{j=1}^{n}b_j\mathbf y_j.
\label{eq:stage1-pose-centroids}
\end{equation}
For any fixed rotation $\mathbf R$, the optimal translation aligns these two centroids:
\[
\mathbf t = \bar{\mathbf x} - \bar{\mathbf y}\mathbf R^\top.
\]
We can therefore estimate the rotation using the centered coordinates
\[
\widetilde{\mathbf x}_i = \mathbf x_i-\bar{\mathbf x},
\qquad
\widetilde{\mathbf y}_j = \mathbf y_j-\bar{\mathbf y}.
\]

\paragraph{Rotation and translation.}
Using the centered points, define the weighted cross-covariance matrix
\begin{equation}
\mathbf H = \sum_{i=1}^{m}\sum_{j=1}^{n}
T^{(1)}_{ij} \widetilde{\mathbf y}_j^\top \widetilde{\mathbf x}_i \in\mathbb R^{3\times3}.
\label{eq:stage1-pose-covariance}
\end{equation}
This matrix summarizes how the centered template and scene points are aligned under the Stage I correspondence weights. Substituting the optimal translation into Eq.~\ref{eq:stage1-pose-fit} reduces the problem to finding the rotation that maximizes $\operatorname{tr}(\mathbf R\mathbf H)$ over $\mathbf R\in\mathrm{SO}(3)$. We solve this weighted alignment problem using the Kabsch algorithm~\citep{kabsch1976rotation}.

Let
\[
\mathbf H = \mathbf U\boldsymbol\Sigma\mathbf V^\top
\]
be a singular value decomposition, and define
\begin{equation}
\mathbf J = \operatorname{diag}
\left( 1,1,\det(\mathbf V\mathbf U^\top)
\right).
\end{equation}
The estimated rotation and translation are then
\begin{equation}
\mathbf R_0 = \mathbf V\mathbf J\mathbf U^\top,
\qquad \mathbf t_0
= \bar{\mathbf x} - \bar{\mathbf y}\mathbf R_0^\top.
\label{eq:stage1-pose-solution}
\end{equation}
The matrix $\mathbf J$ prevents the solution from introducing a reflection, so that $\mathbf R_0$ is a valid rotation.

\paragraph{Use in Stage II.}
We compute $(\mathbf R_0,\mathbf t_0)$ once after Stage I and hold it fixed throughout Stage II. The fitted pose is used to raycast the aligned template and construct the visibility reference $\mathbf q_{\mathrm{ray}}$, and it also enters the pose-conditioned cross-cost $\mathbf M^{\mathrm{pose}}$. Stage II therefore refines
template visibility and correspondence without re-estimating the rigid pose. Computing the alignment requires only the weighted centroids, the $3\times3$ matrix $\mathbf H$, and a $3\times3$ singular value decomposition; it does not require an additional GWD or Sinkhorn optimization.

\subsection{KL-regularized Stage II and generalized Sinkhorn updates}
\label{app:kl-sinkhorn}
We describe how the KL-divergence regularizer in Stage II modifies the Sinkhorn iterations. Stage II fixes the predicted scene distribution $\widehat{\mathbf p}$ and learns the template distribution
\[
\mathbf q=\mathbf T^\top\one_m.
\]
As in Sec.~4.3, we use the raycast-based reference $\mathbf r=0.98\mathbf q_{\mathrm{ray}}+0.02\mathbf b$. With $\mathbf q_{\mathrm{ray}}$ and $\mathbf b$ normalized and $\mathbf b$ strictly positive, $\mathbf r$ is a strictly positive probability vector. Thus, the KL penalty encourages the learned template distribution to
remain close to the raycast prediction, while the positive reference weights ($r_j>0$) allow every template point to receive correspondence mass.

For probability vectors $\mathbf q,\mathbf r\in\Delta_n$, we use
\begin{equation}
\KL(\mathbf q\|\mathbf r) = \sum_{j=1}^{n} \left( q_j\log\frac{q_j}{r_j}-q_j+r_j \right)
= \sum_{j=1}^{n} q_j\log\frac{q_j}{r_j},
\label{eq:appendix-kl-definition}
\end{equation}
where $0\log 0=0$. The second equality follows because
$\sum_j q_j=\sum_j r_j=1$. The coefficient $\tau$ controls the strength of this soft visibility prior: larger values of $\tau$ encourage the learned marginal
$\mathbf q$ to remain closer to $\mathbf r$.

\paragraph{Inner optimization problem.}
We use the two-level iteration described in Sec.~2.2. As defined in the main body, at outer iteration $k$, its Stage-II specialization is
\begin{equation}
\begin{aligned}
\mathbf q^{(k)}
&=(\mathbf T^{(k)})^\top\one_m,\\
\mathbf D_{\mathbf T^{(k)}}
&=(\mathbf C_X^{\circ 2}\widehat{\mathbf p})\one_n^\top+
\one_m
(\mathbf C_Y^{\circ 2}\mathbf q^{(k)})^\top-
2\mathbf C_X\mathbf T^{(k)}\mathbf C_Y^\top,\\
\mathbf Q_{\mathbf T^{(k)}}
&=\lambda_G\mathbf D_{\mathbf T^{(k)}}
+\lambda_M\mathbf M^{\mathrm{pose}},\\
\mathbf K_{\mathbf T^{(k)}}
&=\exp\!\left(-\frac{\mathbf Q_{\mathbf T^{(k)}}}{\varepsilon}\right).
\end{aligned}
\label{eq:appendix-kl-kernel}
\end{equation}
 Here $\mathbf D_{\mathbf T^{(k)}}$ is the gradient of the loss $\mathcal G$ in Eq.~\ref{eq:gw-half-loss}, with $\mathbf C_X$ and $\mathbf C_Y$ being the normalized shortest-path distances defined in Eq.~\ref{eq:template-shortest-path-cost}. 

We hold $\mathbf Q_{\mathbf T^{(k)}}$ fixed during the inner iterations and solve
\begin{equation}
\min_{\substack{\mathbf T\geq0\\
\mathbf T\one_n=\widehat{\mathbf p}}}
\left\langle
\mathbf Q_{\mathbf T^{(k)}},\mathbf T
\right\rangle + \varepsilon
\sum_{i,j}T_{ij}(\log T_{ij}-1) + \tau \KL(\mathbf T^\top\one_m\|\mathbf r).
\label{eq:appendix-kl-inner}
\end{equation}
At each outer iteration, the nonlinear structural GWD term is replaced by its first-order approximation around the current coupling $\mathbf T^{(k)}$. The entropy and KL terms are kept in their original form in the inner problem. This gives a semi-relaxed optimization: the scene marginal $\widehat{\mathbf p}$ is fixed exactly, while the template marginal is learned and softly regularized toward $\mathbf r$. The derivation here concerns this formulation without additional capacity or affine-moment constraints.

\paragraph{Generalized Sinkhorn updates.}
We now derive how the KL penalty modifies the usual Sinkhorn column scaling. The updates are summarized in the following lemma.

\begin{lemma}
Suppose $\mathbf Q_{\mathbf T^{(k)}}$ is finite,
$\varepsilon>0$, $\tau>0$, and $\widehat{\mathbf p}$ and $\mathbf r$ are strictly positive probability vectors. Then Eq.~\ref{eq:appendix-kl-inner} has a unique minimizer of the form
\begin{equation}
\mathbf T = \operatorname{diag}(\mathbf u)
\mathbf K_{\mathbf T^{(k)}}
\operatorname{diag}(\mathbf v),
\label{eq:appendix-kl-factorization}
\end{equation}
where
\begin{equation}
\mathbf u =
\frac{\widehat{\mathbf p}}
{\mathbf K_{\mathbf T^{(k)}}\mathbf v},
\qquad
\mathbf v =
\left(
\frac{\mathbf r}
{\mathbf K_{\mathbf T^{(k)}}^\top\mathbf u}
\right)^{\theta},
\qquad
\theta =
\frac{\tau}{\tau+\varepsilon}.
\label{eq:appendix-kl-scaling}
\end{equation}
All divisions and powers are applied elementwise.
\end{lemma}

\begin{proof}
The feasible set is nonempty and compact, and the entropy term makes the inner objective strictly convex, so the minimizer is unique. Introducing a multiplier $\eta_i$ for each fixed row sum, stationarity at a positive optimum gives
\begin{equation}
Q_{ij} + \varepsilon\log T_{ij}
+ \tau\log\frac{q_j}{r_j} +
\eta_i = 0,
\end{equation}
where
$\mathbf q=\mathbf T^\top\one_m$ and
$\mathbf Q=\mathbf Q_{\mathbf T^{(k)}}$.
Hence
\begin{equation}
T_{ij} = u_i
(\mathbf K_{\mathbf T^{(k)}})_{ij} \left(
\frac{r_j}{q_j} \right)^{\tau/\varepsilon},
\label{eq:appendix-kl-factor}
\end{equation}
with $u_i=\exp(-\eta_i/\varepsilon)$. Writing
\[
v_j=\left(\frac{r_j}{q_j}\right)^{\tau/\varepsilon}
\]
gives
$\mathbf T=\operatorname{diag}(\mathbf u) \mathbf K_{\mathbf T^{(k)}} \operatorname{diag}(\mathbf v)$.
Since
\[
q_j = v_j (\mathbf K_{\mathbf T^{(k)}}^\top\mathbf u)_j,
\]
we obtain
\begin{equation}
v_j^{\,1+\tau/\varepsilon}
= \left( \frac{r_j}
{(\mathbf K_{\mathbf T^{(k)}}^\top\mathbf u)_j} \right)^{\tau/\varepsilon},
\end{equation}
which yields the exponent $\theta=\tau/(\tau+\varepsilon)$ in Eq.~\ref{eq:appendix-kl-scaling}. The fixed row marginal gives the stated update for $\mathbf u$.
\end{proof}

Starting from $\mathbf v^{(0)}=\one_n$, the inner iterations alternate
\begin{equation}
\mathbf u^{(\ell+1)} =
\frac{\widehat{\mathbf p}}
{\mathbf K_{\mathbf T^{(k)}}\mathbf v^{(\ell)}},
\qquad
\mathbf v^{(\ell+1)} =
\left(\frac{\mathbf r}
{\mathbf K_{\mathbf T^{(k)}}^\top \mathbf u^{(\ell+1)}}
\right)^{\theta}.
\label{eq:appendix-kl-iterations}
\end{equation}
The row update enforces the fixed scene marginal, while the column update incorporates the soft prior on template visibility. The resulting template marginal is
\begin{equation}
\mathbf q=
\mathbf v\odot
\left(\mathbf K_{\mathbf T^{(k)}}^\top\mathbf u
\right).
\end{equation}
For $\tau=0$, $\theta=0$ and $\mathbf v=\one_n$, so the reference distribution has no effect. As $\tau$ increases, the learned marginal is pulled more strongly toward $\mathbf r$. In the extreme case, as $\tau\rightarrow\infty$, the update approaches the balanced scaling that fixes the template marginal to $\mathbf r$.

\paragraph{EGGroW computations.}
We next show that the generalized Sinkhorn updates above can be implemented within the same EGGroW framework used for the standard entropic problem. The KL penalty changes only the scaling updates and does not alter the two kernel--vector products required at each inner iteration. DET-EGGroW evaluates the structural matrix operations used to construct $\mathbf Q_{\mathbf T^{(k)}}$ using the deterministic methods of Sec.~3.1.

For APP-EGGroW, let the positive random-feature approximation of the Stage-II kernel be
\begin{equation}
\widehat{\mathbf K}_{\mathbf T^{(k)}} =
\operatorname{diag}(\mathbf s) \Phi_X^\top\Phi_Y
\operatorname{diag}(\mathbf t).
\label{eq:appendix-kl-feature-kernel}
\end{equation}
The products are then evaluated as
\begin{equation}
\begin{aligned}
\widehat{\mathbf K}_{\mathbf T^{(k)}}\mathbf v &=
\mathbf s\odot \Phi_X^\top
\bigl( \Phi_Y(\mathbf t\odot\mathbf v)
\bigr),\\
\widehat{\mathbf K}_{\mathbf T^{(k)}}^\top\mathbf u &=
\mathbf t\odot
\Phi_Y^\top \bigl( \Phi_X(\mathbf s\odot\mathbf u)
\bigr).
\end{aligned}
\label{eq:appendix-kl-feature-products}
\end{equation}
The features and scale factors are constructed using the Stage-II coefficients $\lambda_G$ and $\lambda_M$ in place of $\alpha$ and $1-\alpha$, respectively. The resulting approximate coupling remains factorized:
\begin{equation}
\widehat{\mathbf T} =
\mathbf L^\top\mathbf R,
\qquad
\mathbf L =
\Phi_X
\operatorname{diag}(\mathbf s\odot\mathbf u),
\qquad
\mathbf R =
\Phi_Y
\operatorname{diag}(\mathbf t\odot\mathbf v).
\label{eq:appendix-kl-coupling-factors}
\end{equation}
Consistent with Sec.~3.2, we retain both scale factors. In particular, the template-side factor $\mathbf t$ affects the learned marginal $\mathbf q$ and therefore cannot be absorbed or discarded. The KL penalty only modifies the scaling updates and does not require additional random features or a larger factorization.

\paragraph{Numerical implementation.}
We implement the generalized Sinkhorn updates in a numerically stable form and explicitly enforce the fixed scene marginal before returning each inner solution. For the dense implementation, we evaluate the scaling updates in the log domain. Writing $\operatorname{LSE}$ for log-sum-exp,
\begin{equation}
\begin{aligned}
\log u_i &=
\log\widehat p_i - \operatorname{LSE}_j
\left( -\frac{Q_{ij}}{\varepsilon} +\log v_j \right),\\
\log v_j
&= \theta \left[ \log r_j -
\operatorname{LSE}_i \left(-\frac{Q_{ij}}{\varepsilon}
+\log u_i\right)\right].
\end{aligned}
\label{eq:appendix-kl-log-scaling}
\end{equation}
These are the log-domain forms of the updates in
Eq.~\eqref{eq:appendix-kl-scaling} and avoid numerical underflow or overflow in the exponential kernel. APP-EGGroW evaluates the same kernel--vector products through the factorization in Eq.~\eqref{eq:appendix-kl-feature-products}, while retaining the associated scale factors.

During the inner iterations, we monitor the row-marginal residual and the change in the scaling variables. After the final column update, we recompute
\[
\mathbf u = \frac{\widehat{\mathbf p}}
{\mathbf K_{\mathbf T^{(k)}}\mathbf v}
\]
so that the returned coupling satisfies the fixed scene marginal $\mathbf T\one_n=\widehat{\mathbf p}$. If $\widehat p_i=0$, the corresponding row is omitted from the log-domain scaling iterations and restored as a zero row afterward. The resulting coupling defines the next outer iterate $\mathbf T^{(k+1)}$, which is used to recompute $\mathbf D_{\mathbf T^{(k+1)}}$ and construct the next Stage-II kernel.

\section{Additional Numerical Results}
\label{sec:additional-numerical-results}

\subsection{Random-feature ablation of APP-EGGroW}
\label{app:rf-ablation}

We compare five positive random-feature constructions for APP-EGGroW on
FAUST $004\to007$ (Table~\ref{tab:faust-rf-pose-ablation}). For
$n\in\{1{,}000,2{,}000\}$, each construction approximates one balanced
Sinkhorn update at the dense GWD iterates $T^{(20)}$ and $T^{(50)}$,
with $r=512$, $\alpha=1$, $\varepsilon=0.05$, and centered diagonal
conditioning. Note that these are dense GWD only, whereas in the Sec.~\ref{sec:exp-pose} we approximated FGWD using APP-EGGroW instead.

Structural costs are full-mesh graph shortest paths with Euclidean edge
weights, jointly scaled by the larger sampled 95th percentile. Source and
target use matching vertex samples and uniform marginals. Sinkhorn uses
tolerance $10^{-9}$ and at most 500 iterations; correspondences are obtained
by row-wise maximization of the coupling.

\begin{table}[!ht]
\centering
\small
\setlength{\tabcolsep}{4pt}
\renewcommand{\arraystretch}{1.08}
\begin{tabular}{lcccc}
\hline
Random-feature & \multicolumn{2}{c}{$n=1{,}000$} & \multicolumn{2}{c}{$n=2{,}000$} \\
\hline
construction & Geodesic error $\downarrow$ & PCK@5\% $\uparrow$ & Geodesic error $\downarrow$ & PCK@5\% $\uparrow$ \\
\hline
IID PRF & $0.310\pm0.082$ & $22.50\pm3.16$ & $0.465\pm0.029$ & $18.88\pm3.82$ \\
OPRF-IID & $0.325\pm0.092$ & $21.84\pm2.57$ & $0.457\pm0.036$ & $19.09\pm4.13$ \\
Orthogonal PRF & $0.421\pm0.044$ & $20.79\pm2.71$ & $0.430\pm0.113$ & $17.95\pm1.78$ \\
Simplex PRF & $0.422\pm0.044$ & $20.80\pm2.70$ & $0.430\pm0.113$ & $17.97\pm1.77$ \\
OPRF-Orthogonal & $0.424\pm0.048$ & $20.53\pm2.75$ & $0.433\pm0.116$ & $17.78\pm1.87$ \\
\hline
Dense-kernel reference & $0.290$ & $25.15$ & $0.358$ & $22.58$ \\
\hline
\end{tabular}
\caption{Random-feature ablation on FAUST. RF entries are mean $\pm$ sample
standard deviation over five/three seeds at $n=1{,}000/2{,}000$, after
averaging the two frozen updates. Geodesic error is the mean raw target
distance; PCK@5\% is the percentage within 5\% of the sampled target
geodesic diameter. The dense reference averages the same states.
OPRF denotes optimized positive random features.}
\label{tab:faust-rf-pose-ablation}
\end{table}

Among RF variants, IID features achieve the lowest mean geodesic error and
highest PCK at $n=1{,}000$. At $n=2{,}000$, orthogonal and simplex features
give lower mean geodesic errors, whereas OPRF-IID achieves the highest PCK.

\subsection{Pose Correspondence}
\label{sec:pose-correspondence-appendix}
To complement the pose correspondence experiments in Sec.~\ref{sec:exp-pose}, Fig.~\ref{fig:faust-gw-correspondence-006-007} presents additional qualitative results from source scan $004$ to target scans $006$ and $007$. Similar to Fig.~\ref{fig:faust-gw-correspondence}, transferring intrinsic source colors to each target visualizes how well each method preserves correspondence across changes in pose. The additional poses highlight the challenges introduced by articulation, particularly when distinct limbs become close in Euclidean space despite remaining distant along the mesh. We used \(\varepsilon=0.05\) and ran each experiment for 50 outer updates with 500 inner Sinkhorn iterations per update.

Regarding further experiment setup details, we normalized structural distances by the full-source off-diagonal geodesic 95th percentile, centered each mesh, and scaled both coordinate sets by the source RMS radius. Numerical correspondences were extracted through row-wise maximization of the coupling, with normalized errors measured relative to the full target diameter and shaded bands showing the observed minimum--maximum across runs. Moreover, we timed method-specific setup and solving, excluding shared calibration, input loading, and post-solve scoring.

\begin{figure}[!ht]
\centering
\includegraphics[width=\linewidth]{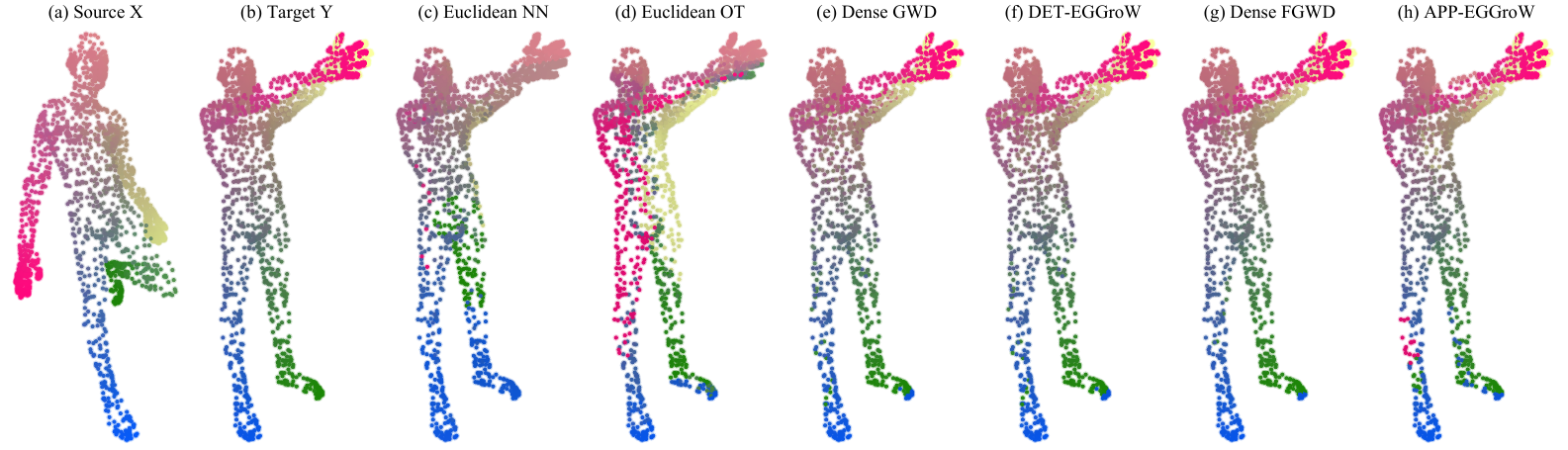}
\includegraphics[width=\linewidth]{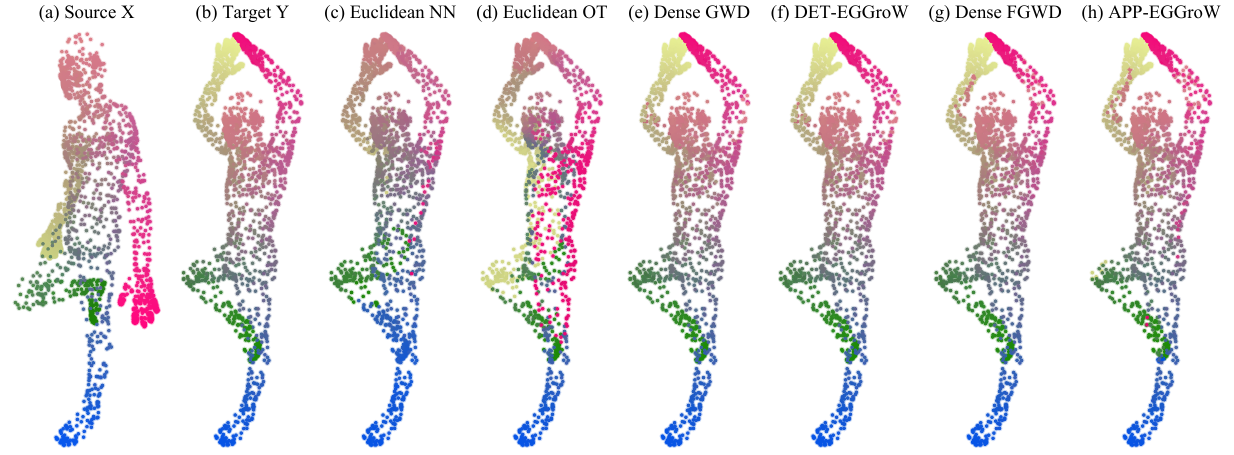}
\caption{FAUST pose correspondence for (a) source scan $004$ and (b) target scans $006$ (top) and $007$ (bottom). Colors are intrinsic source colors transferred to the target by each method (c)-(h).}
\label{fig:faust-gw-correspondence-006-007}
\end{figure}

\subsection{Template Detection}
\label{sec:template-detection-appendix}

Here, we provide specific details on the setup for template detection in Sec.~\ref{sec:exp-template}. We used meshes from the Thingi10k dataset \cite{Zhou2016Thingi10K} with the following IDs: \texttt{[977256, 104968, 73083, 718339, 38644, 44381, 134543, 65444, 39946, 985159, 263199, 931879, 42370, 69987, 71999, 112544, 98020, 53750, 1514903, 362964]}.

We center each source mesh at the midpoint of its bounding box and normalize its bounding-box diagonal to one. This normalization remains fixed across support sizes and is also used for the target. For the sampled-point experiments, we draw points uniformly with respect to source surface area and use nested prefixes of the same sample sequence to vary resolution. Each support contains $n$ points with uniform masses $a_i=b_i=1/n$, and all methods receive identical inputs at each size.

We construct the target by applying a smooth nonrigid deformation, followed by a fixed rotation and translation, to the sampled source points. A random permutation removes their original ordering while preserving known material correspondences for evaluation. For shortest-path comparisons, we construct an undirected graph by taking the union of the directed $8$-nearest-neighbor edges, weighted by Euclidean length. Disconnected components are joined using Euclidean minimum-spanning-tree bridges.


\subsection{Limitations and Future Work}
\label{sec:limitations-future-work}
While our experiments cover a selected range of poses, object geometries, and scenes, a broader evaluation would help establish robustness across subjects and acquisition conditions. Template detection also depends on the quality of the observed geometry and initialization, particularly under clutter and occlusion. Future work will explore more diverse datasets, stronger geometric features, and adaptive approximation strategies to improve robustness and scalability for larger, more complex scenes.

\end{document}